\documentclass{article}

\usepackage{mathtools, amssymb, amsthm, xspace, mathabx, verbatim, authblk}  

\usepackage{float, graphicx}

\usepackage{todonotes}

\newtheorem{thm}{Theorem}[section]

\newtheorem{lem}[thm]{Lemma}
\newtheorem{cor}[thm]{Corollary}
\newtheorem{defn}{Definition}[section]

\numberwithin{equation}{section}
\usepackage[figurewithin=section]{caption}  

\newcommand{\Lov}{Lov\'{a}sz\xspace}

\newcommand{\R}{\ensuremath{\mathbb{R}}\xspace}

\newcommand{\mP}{\ensuremath{\mathcal{P}}\xspace}

\newcommand{\mA}{\ensuremath{\mathcal{A}}\xspace}
\newcommand{\mB}{\ensuremath{\mathcal{B}}\xspace}
\newcommand{\mD}{\ensuremath{\mathcal{D}}\xspace}
\newcommand{\mE}{\ensuremath{\mathcal{E}}\xspace}
\newcommand{\mK}{\ensuremath{\mathcal{K}}\xspace}
\newcommand{\mH}{\ensuremath{\mathcal{H}}\xspace}
\newcommand{\mS}{\ensuremath{\mathcal{S}}\xspace}
\newcommand{\mPosDef}{\ensuremath{\mS_{++}^n}\xspace}
\newcommand{\mX}{\ensuremath{\mathcal{X}}\xspace}
\newcommand{\E}{\ensuremath{\mathbb{E}}\xspace}
\renewcommand{\P}{\ensuremath{\mathbb{P}}\xspace}

\newcommand{\set}[1]{\ensuremath{\left\{#1\right\}}\xspace}
\newcommand{\st}{\ensuremath{\;\big|\;}\xspace}
\newcommand{\Otilde}{\ensuremath{\widetilde{O}}\xspace}
\newcommand{\tr}{\mathrm{tr}}

\newcommand{\ip}[2]{\left \langle #1, \, #2 \right \rangle}
\newcommand{\ta}{\ensuremath{\tilde{a}}\xspace}

\title{John's Walk}
\author[1]{Adam Gustafson \thanks{adam.marc.gustafson@gmail.com}}
\author[2]{Hariharan Narayanan \thanks{hariharan.narayanan@tifr.res.in}}
\affil[1]{Microsoft Corporation}
\affil[2]{ School of Technology and Computer Science, Tata Institute of Fundamental Research}

\date{\today}

\begin{document}
\maketitle

\begin{abstract}
We present an affine-invariant random walk for drawing uniform random samples from a convex body $\mK \subset \R^n$ that uses maximum volume inscribed ellipsoids, known as John's ellipsoids, for the proposal distribution.  
Our algorithm makes steps using uniform sampling from the John's ellipsoid of the symmetrization of $\mK$ at the current point. We show that from a warm start, the random walk mixes in $\Otilde(n^7)$ steps\footnote{$\Otilde(\cdot)$ notation suppresses polylogarithmic factors as well as constants depending only on the error parameters.} where the log factors depend only on constants associated with the warm start and desired total variation distance to uniformity.  We also prove polynomial mixing bounds starting from any fixed point $x$ such that for any chord $pq$ of $\mK$ containing $x$, $\left|\log \frac{|p-x|}{|q-x|}\right|$ is bounded above by a polynomial in $n$.
\end{abstract}

\section{Introduction}
Drawing random samples from a convex body in $\mK \subset \R^n$ is an important problem for volume computation and optimization which has generated a large body of research.  Usually $\mK$ is specified by a membership oracle which certifies whether or not a test point $x \in \R^n$ is contained in $\mK$.  Given such an oracle, geometric random walks are then used to explore $\mK$ such that after a sufficient number of steps, the walk has ``mixed'' in the sense that the current point is suitably close to a point uniformly drawn from $\mK$ in terms of statistical distance.   To use such walks, an assumption that $\mB(r) \subset \mK \subset \mB(R)$ is often made, where $\mB(r)$ represents the Euclidean ball of radius $r > 0$.  One common example of such geometric walks is the Ball Walk, which generates the next point by uniformly randomly sampling from a ball of radius $\delta \leq r/\sqrt{n}$ centered at the current point, and mixes in
$\Otilde(n(R^2/\delta^2))$ steps from a warm start (i.e., the starting distribution has a density bounded above by a constant) \cite{kannan1997random}.    Another is Hit and Run, where the next point is chosen uniformly at random from a random chord in $\mK$ which intersects the current point.  Hit and Run mixes in $O(n^3(R/r)^2 \log(R/(d\epsilon)))$ where the starting point is a distance $d$ from the boundary and $\epsilon$ is the desired distance to stationarity \cite{lovasz2006hit}. 
Affine-invariant walks (i.e., geometric walks whose mixing time is invariant to such affine transformations) are another class of random walks which avoid the problem of rounding.  One such random walk is known as Dikin Walk \cite{kannan2012random}, which uses uniform sampling from Dikin ellipsoids to make steps.  Given a polytope with $m$ inequality constraints, the Dikin Walk mixes in $\Otilde(mn)$ steps from a warm start.  This random walk was extended to general convex bodies equipped with a $\nu$-self-concordant barrier in \cite{narayanan2016randomized}, and mixes in $\Otilde(n^3\nu^2)$ steps from a warm start.  For the case of a polytope, this implies that the Dikin walk equipped with the Lee-Sidford (LS) barrier \cite{lee2013path} mixes in $\Otilde(n^5)$ steps from a warm start, though at each step one must additionally compute the LS barrier which requires $O(nnz(A) + n^2)$ arithmetic operations, where $nnz(A)$ is the number of non-zeros in the matrix $A$ which defines the polytope.  A significantly improved analysis of this walk was performed by \cite{chen2017fast}, and their algorithm reaches a total variation distance of $\epsilon$ from the uniform measure in $O\left(n^{2.5}\log^4\left(\frac{2m}{n}\right)\log\left(\frac{M}{\epsilon}\right)\right)$
steps from an $M$-warm start. 
Very recently, it was shown in \cite{Laddha} that for certain ``strongly self-concordant" barriers
there is a Dikin walk that mixes in $\tilde{O}(n \tilde{\nu})$ where $\tilde{\nu}$ is related to the self-concordence parameter of the barrier.

This paper introduces another affine-invariant random walk akin to Dikin Walk which uses uniform sampling from John's ellipsoids of a certain small radius of appropriately symmetrized convex sets to make steps, and show that this walk mixes to within a total variation distance $\epsilon$ in $O(n^7\log \epsilon^{-1})$ steps from a warm start.   The type of convex body $\mK$ is not specified (i.e., need not be a polytope) in our analysis of the mixing time, but one must have access to the John's ellipsoid of the current symmetrization of the convex body.  While this dependence on the dimension is admittedly steep, a significant  feature of this walk is that its  mixing time from a warm start or alternatively ``central point" such as the center of mass, can be bounded above by a quantity that has absolutely no dependence on any parameter associated with the body apart from its dimension. 

\noindent{\bf Notation:}
We will denote a large universal constant by $C$ and a small universal constant by $c$.

\noindent Our main theorems at the end of this paper are the following.

\noindent{\bf Theorem 1:}
Let  $\epsilon > 0$ and $M =\sup \frac{\pi_0(A)}{\pi(A)}$, after $t(\epsilon) =  C n^7 \log({M}/\epsilon) $ steps of John's walk, we have $d_{TV}(\pi_{t(\epsilon)}, \pi) \leq \epsilon$.  

\noindent{\bf Theorem 2:}
For all chords $pq$ of $\mK$ containing $x$, assume $\frac{|p-x|}{|q-x|} \in (\eta, \eta^{-1})$ for some parameter $0 < \eta < 1$ that measures the centrality of $x$ in $\mK$. Then, there is a random geometrically distributed time $\tau$ with mean bounded above by $C$ such that 
for $\epsilon > 0$, after $t(\epsilon) + \tau = C n^7 \left(n \log(\sqrt{n}/(r\eta)) + \log({1}/\epsilon)\right) + \tau$ steps of John's walk starting at $x$, we have $d_{TV}(\pi_{t(\epsilon)+ \tau}, \pi) \leq \epsilon$.  \\\\

It is known that for the center of mass, $\eta \geq \frac{c}{n}.$

\section{John's Walk}\label{sec:john_main}
In this section, we describe John's maximum volume ellipsoid for a convex body $\mK \subset \R^n$, and describe a geometric random walk using such ellipsoids.  We begin with reviewing John's theorem and some implications of the theorem.

\subsection{John's Theorem}
Fritz John showed that any convex body contains a unique ellipsoid of maximal volume, and characterized the ellipsoid \cite{john48, ball1992ellipsoids}.  Without loss of generality, we may assume that the ellipsoid of maximal volume is the unit Euclidean ball $\mB \subset \R^n$, since this is a case after an affine transformation.  John's theorem as stated for the unit ball case is as follows:
\begin{thm}[John's Theorem]\label{thm:john}
Each convex body $\mK \subset \R^n$ contains a unique ellipsoid of maximal volume.  The ellipsoid is $\mB$ if and only if the following conditions are satisfied: $\mB \subset \mK$, and for some $m \geq n$ there are Euclidean unit vectors $\set{u_i}_{i=1}^m$ on the boundary of $\mK$ and positive constants $\set{c_i}_{i=1}^m$ satisfying,
\begin{align}
&\sum_{i=1}^m c_i u_i = 0, \label{eq:john_cond1} \\  
&\sum_{i=1}^m c_i u_i u_i^T = I_n, \label{eq:john_cond2} 
\end{align}
where $I_n$ denotes the identity matrix in $\R^{n\times n}$.  
\end{thm}
\noindent  Note that condition \eqref{eq:john_cond2} is sometimes written equivalently as
$$
\ip{x}{y} = \sum_{i=1}^m c_i \ip{u_i}{x} \ip{u_i}{y}
$$
for all $x, y \in \R^n$.  Using the cyclic invariance of the trace and that the $\set{u_i}$ are unit vectors, condition \eqref{eq:john_cond2} implies that
\begin{equation}\label{eq:john_csum}
\sum_{i=1}^n c_i = n,
\end{equation}
a property we employ in subsequent analysis.
We now enumerate some properties from \cite{ball1992ellipsoids} which provide additional insight into the geometric properties of John's ellipsoids and are useful for the analysis in subsequent sections. Note that condition \eqref{eq:john_cond1} implies that all the contact points do not lie in one in one half-space of the unit ball, and this condition is redundant in the symmetric case, since for every contact point $u_i$, its reflection about the origin $-u_i$ is also a contact point.  Condition \eqref{eq:john_cond2} guarantees such contact points do not lie close to a proper subspace. Furthermore, there are at most $n(n+3)/2$ contact points for general $\mK$, and $n(n+1)/2$ non-redundant contact points if $\mK$ is origin-symmetric \cite{gruber1988minimal}.  At each $u_i$, the supporting hyperplane to $\mK$ is unique and orthogonal to $u_i$, since this is the case for the unit ball.  Thus considering the polytope resulting from such supporting hyperplanes, $\mP = \set{x \in \R^n \st \langle x, u_i \rangle \leq 1, \, i = 1, \ldots, m}$, the convex set $\mK$ obeys the sandwiching $\mB \subset \mK \subset \mP$.  By Cauchy-Schwarz, for any $x \in \mP$, we have
$$
-|x| \leq \ip{u_i}{x} \leq 1.
$$
Since the weights $\set{c_i}$ are positive, it follows by employing conditions \eqref{eq:john_cond1}, \eqref{eq:john_cond2}, and \eqref{eq:john_csum} that
$$
\begin{aligned}
0 &\leq \sum_i c_i (1 - \ip{u_i}{x})(|x| + \ip{u_i}{x}) \\
&= |x|\sum_i c_i + (1 - |x|) \ip{\sum_i c_i u_i}{x} - \sum_i c_i \ip{u_i}{x}^2 \\
&= n |x| - |x|^2,
\end{aligned}
$$
from which it follows that $|x| \leq n$. 
 If the convex body is origin-symmetric, then by substituting $-u_i$ for $u_i$, for any $x \in \mP$, we have
$$
|\ip{u_i}{x}| \leq 1.
$$
It follows that 
$$
|x|^2 = \sum_{i=1}^m c_i \ip{u_i}{x}^2 \leq \sum_{i=1}^m c_i = n,
$$
so $|x| \leq \sqrt{n}$.  
It is known that if $\mK$ is origin-symmetric and the unit ball is the John's ellipsoid, then the containment is
\begin{equation}\label{eq:john_nesting2}
\mB \subset A(\mK) \subset \sqrt{n} \mB.
\end{equation}
\subsection{The John's Walk Algorithm}
We state the algorithm for a general convex body $\mK$.  At a given point $x \in \mK$, let the symmetrization of $x \in \mK$ be  
$$
\mK_x^s \equiv \mK \cap \set{2x - y \st y \in \mK},
$$ 
and let $\mE_x = \set{E_x u + x \st |u| \leq 1}$ denote the John's ellipsoid of $\mK_x^s$.  Similarly, let the rescaled John's ellipsoid be $\mE_x(r) = \set{r (E_x u) + x \st |u| \leq 1}$, where the radius $r > 0$ will be specified in section \ref{sec:mixing}.  Assume $0 = x_0 \in \text{int}(\mK)$, and we have computed $\mE_{x_0}$.  To generate a sample $x_i$ given $x_{i-1}$, we use algorithm \ref{alg:johns_walk}, where $\lambda(\cdot)$ denotes the Lebesgue measure on $\R^n$:

\begin{figure}[H]
\renewcommand{\figurename}{Algorithm}
\noindent
\begin{minipage}{\textwidth}
\noindent\rule{\linewidth}{.4 pt}
\caption{John's Walk Step}
\label{alg:johns_walk}
\noindent\rule{\linewidth}{.4 pt}
Given $x \in \mK_x^s$, $r > 0$, and $\mE_x$, generate the next step $y$ as follows:
\begin{enumerate}
\item Toss a fair coin.  If the result is heads, let $y = x$.  
\item If the result is tails:
\begin{enumerate}
\item Draw a uniformly distributed random point $z$ from $\mE_x(r)$.
\item Compute $\mE_z$ using $\mK_z^s$.
\item If $x \notin \mE_z(r)$, let $y = x$. Otherwise, let
$$
y = \begin{cases}
z, & \text{with probability } \min\left(1, \,\frac{\lambda(\mE_x(r))}{\lambda(\mE_z(r))}\right) = \min\left(1,\, \frac{\det E_x}{\det E_z}\right), \\
x, & \text{else}.  
\end{cases}
$$
\end{enumerate}
\end{enumerate}
\noindent\rule{\linewidth}{.4 pt}
\end{minipage}
\renewcommand{\figurename}{Figure}
\end{figure} \noindent 
Algorithm \ref{alg:johns_walk} is a Metropolis-Hastings geometric random walk which uses the uniform measure $Q_x(\cdot)$ on the dilated John's ellipsoid $\mE_x(r)$ as the proposal distribution. Tossing a fair coin ensures the transition probability kernel defined by the algorithm is positive definite, which is known as making the walk lazy. Lazy random walks have the same stationary distribution as the original walk at the cost of a constant increase in mixing time (we will analyze the non-lazy walk, noting that the mixing time is not affected in terms of complexity as a function of $m$ and $n$).  The rejection of any sample $y$ such that $x \notin \mE_y(r)$ is necessary to ensure the random walk is reversible.  

The uniform measure on the John's ellipsoid $\mE_x(r)$ is absolutely continuous with respect to the Lebesgue measure $\lambda$, and thus the Radon-Nikodym derivative (i.e., density) for the proposal distribution is 
\begin{equation}\label{eq:proposal_density}
\begin{aligned}
q_x(y) &\equiv \frac{dQ_x}{d\lambda}(y) \\
&= \left(\frac{1}{\lambda(\mE_x(r))}\right) \cdot 1_{\set{y \in \mE_x(r)}}.
\end{aligned}
\end{equation}
The acceptance probability corresponding to the uniform stationary measure in the Metropolis filter is 
$$
\alpha_x(y) = \min\left[1, \; \frac{\lambda(\mE_x(r))}{\lambda(\mE_y(r))}\right].
$$
By the Lebesgue decomposition, the transition probability measure $P_x(\cdot)$ of the non-lazy version of algorithm \ref{alg:johns_walk} is  absolutely continuous with respect to the measure 
\begin{equation} \label{eq:mu}
\mu \equiv \lambda + \delta_x,
\end{equation} 
where $\delta_x(\cdot)$ is the Dirac measure at $x$ corresponding to a rejected move.  The transition density is thus
\begin{equation} \label{eq:transition_density}
\begin{aligned}
p_x(y) &\equiv \frac{dP_x}{d\mu}(y) \\
&= \alpha_x(y) q_x(y) 1_{\set{y \neq x, \, x \in \mE_y(r)}} + \rho(x) 1_{\set{y = x}} \\
&= \min \left[\frac{1}{\lambda(\mE_x(r))}, \, \frac{1}{\lambda(\mE_y(r))} \right]1_{\set{y \neq x, \, x \in \mE_y(r), \, y \in \mE_x(r)}} + \rho(x) 1_{\set{y = x}},
\end{aligned}
\end{equation}
where $1_{\set{\cdot}}$ is the indicator function and the rejection probability is denoted $\rho(x)$.  We next analyze the mixing time of the walk.

\section{Analysis of Mixing Time}\label{sec:mixing}
In what follows we let a discrete-time, homogeneous Markov chain be the triple $\set{\mK, \mA, P_x(\cdot)}$ along with a distribution $P_0$ for the starting point, where the sample space is the convex body $\mK \subset \R^n$, the measurable sets on $\mK$ are denoted by $\mA$, and $P_x(\cdot)$ denotes the transition measure for any $x \in \mK$.   

\subsection{Conductance and Mixing Times}
We use the approach from \cite{lovasz1993random} of lower-bounding the conductance of the chain to prove mixing times.  The conductance is defined as follows.
\begin{defn}[Conductance]
Let $P$ be a discrete-time homogenous Markov chain with kernel $P_x(\cdot)$ that is reversible with respect to the stationary measure $\pi(\cdot)$.  Given $A \in \mA$ with $0 < \pi(A) < 1$, the conductance of $A$ is defined as
$$
\phi(A) \equiv \frac{\Phi(A)}{\min\set{\pi(A), \pi(\mK\setminus A)}},  
$$
where $\Phi(A) \equiv \int_A P_u(\mK \setminus A) \, d\pi(u)$.  The conductance of the chain is defined as
$$
\phi \equiv \inf \set{\phi(A) \st A \in \mA, \, 0 < \pi(A) \leq 1/2}.
$$
\end{defn}

Recall the total variation distance between two measures $P_1, P_2$ on a measurable space $(\mK, \mA)$ is 
$$
d_{TV}(P_1, P_2) = \sup_{A \in \mA} |P_1(A) - P_2(A)|.
$$
Note that $|P_1(A) - P_2(A)| = |P_1(\mK\setminus A) - P_2(\mK\setminus A)|$, so if the supremum is attained on any $A \in \mA$, then it is attained on $\mK \setminus A \in \mA$ as well. 
If $P_1$ and $P_2$ are both absolutely continuous with respect to a dominating measure $\mu$ and thus have densities $p_1 \equiv \frac{dP_1}{d\mu}$ and $p_2 \equiv \frac{dP_2}{d\mu}$, respectively, the total variation distance may also be written as
\begin{equation}\label{eq:tv_density}
\begin{aligned}
d_{TV}(P_1, P_2) &= \frac{1}{2}\int \big|p_1 - p_2\big| d\mu \\
&= 1 - \int \min (p_1, p_2)\, d\mu \\
&= 1 - \int_{S_1} \left[\min \left(1, \frac{p_2}{p_1}\right)\right] p_1 \, d\mu \\
&=  1 - \E_{P_1} \left[\min\left(1, \, \frac{p_2}{p_1}\right)\right],
\end{aligned}
\end{equation}  
where $S_1 = \set{x \st p_1(x) > 0}$.  Recall that \eqref{eq:tv_density} does not depend on the choice of dominating measure $\mu$ but rather that the densities are correctly specified with respect to the dominating measure.  Additionally, note that the equality is attained on $\set{x \st p_1(x) \geq p_2(x)}$ almost everywhere with respect to $\mu$ (or alternatively on its complement).  The following relationship between conductance and the total variation distance to the stationary measure was proven in \cite{lovasz1993random}.
\begin{thm}[\Lov and Simonovits]
Let $\pi_0$ be the initial distribution for a lazy, reversible Markov chain with conductance $\phi$ and stationary measure $\pi$, and let $\pi_t$ denote the distribution after $t$ steps.  Let $\pi_0$ be an $M$-warm start for $\pi$, i.e., we have $M = \sup_{A \in \mA} \frac{\pi_0(A)}{\pi(A)}$.  Then
\begin{equation}\label{eq:tv_convergence}
d_{TV}(\pi_t, \pi) \leq \sqrt{M}\left(1 - \frac{\phi^2}{2}\right)^t.
\end{equation}
\end{thm}
\noindent As a consequence, we have the following bound on the mixing time.
\begin{cor}\label{cor:mixing}
Given $\epsilon > 0$ and $M \equiv \sup \frac{\pi_0(A)}{\pi(A)}$, after $t(\epsilon) \equiv \lceil \frac{2}{\phi^2} \log(\sqrt{M}/\epsilon) \rceil$ steps of the chain, we have $d_{TV}(\pi_{t(\epsilon)}, \pi) \leq \epsilon$.  Thus the Markov chain mixes in $\Otilde(\phi^{-2})$ steps from a warm start.
\end{cor}
\noindent To find mixing times, it then suffices to lower-bound the conductance $\phi$.

\subsection{Isoperimetry}
The typical means by which one finds lower bounds on the conductance is via isoperimetric inequalites.  We first restate the cross-ratio used in isoperimetric inequality we will employ.

\begin{defn}[Cross-Ratio]
Let $x, y \in \mK$, and let $p, q$ be the end points of a chord in $\mK$ passing through $x, y$ where the cross-ratio is defined to be
$$
\sigma(x, y) = \frac{|x-y||p-q|}{|p - x||y - q|},  
$$
where $|\cdot|$ denotes the Euclidean norm. 

Additionally, for any $S_1, S_2 \subset \mK$, let 
$$
\sigma(S_1, S_2) = \inf_{x \in S_1, y \in S_2} \sigma(x, y).
$$
\end{defn}
\noindent In \cite{lovasz1999hit}, \Lov proved an isoperimetric inequality involving the cross-ratio from which the conductance $\phi$ may be lower-bounded for the special case of the uniform distribution on a convex body $\mK \subset \R^n$.  It was extended to log-concave measures by \Lov and Vempala in \cite{lovasz2007geometry} for which the uniform measure on convex body is a special case. We state the latter result as follows.

\begin{thm}[\Lov and Vempala]\label{thm:isoperimetry}
For any log-concave measure $\pi(\cdot)$ supported on $\mK$ and a partition of $\mK$ into measurable subsets $S_1$, $S_2$, and $S_3 = \mK\setminus(S_1 \cup S_2)$, we have
\begin{equation}\label{eq:isoperimetry}
\pi(S_3) \geq \sigma(S_1, S_2) \pi(S_1)\pi(S_2).
\end{equation}
\end{thm}

\subsection{Mixing of John's Walk}
\noindent The key step in proving conductance lower bounds is to show that if two points are close in geometric distance, then they are close in statistical distance.  Note that given John's ellipsoid $\mE_x = \set{E_x u + x \st |u| \leq 1}$, a local norm is induced via 
$$
\|y - x\|_x^2 = (y-x)^T E_x^{-2} (y - x).
$$
We first relate this local norm to the cross-ratio as follows.
\begin{thm}\label{thm:norm_bound}
Let $\|\cdot\|_x$ denote the norm induced by the John's ellipsoid of $\mK_x^s$.  Then
$$
\sigma(x, y) \geq \frac{1}{\sqrt{n}}\|y- x\|_x.
$$
\end{thm}

\begin{proof}
Noting that the cross-ratio is invariant to affine transformations, without loss of generality we may assume by a suitable affine transformation that the John's ellipsoid of $\mK_x^s$ is the unit ball, and thus $\|y - x\|_x = |y- x|$. Let $p, x, y, q$ denote successive points on a chord through $\mK_x^s$.  Then 
$$
\begin{aligned}
\sigma(x,y) &= \frac{|x-y||p-q|}{|p-x||y-q|} \\
&\geq \frac{|x-y|\left(|p-x| + |y - q|\right)}{|p-x||y-q|} \\
&\geq \max \left(\frac{|x-y|}{|y-q|}\, ,\, \frac{|x-y|}{|p-x|} \right) \\
&\geq \frac{|x-y|}{\sqrt{n}},
\end{aligned}
$$
where the last inequality follows from the containment in equation \eqref{eq:john_nesting2}.
\end{proof} 

Before bounding the statistical distance between $P_x$ and $P_y$ given a bound on the geometric distance between $x$ and $y$, we first state some useful lemmas regarding the ellipsoids $\mE_x$ and $\mE_y$.  The next lemma is a generalization of the Cauchy-Schwarz inequality to semidefinite matrices.
\begin{lem}[Semidefinite Cauchy-Schwarz] \label{lem:semidef_cs}  Let $\alpha_1, \ldots, \alpha_m \in \R$ and let $A_1, \ldots, A_m \in \R^{r \times n}$.  Then
\begin{equation}
\left(\sum_{i=1}^m \alpha_i A_i\right) \left(\sum_{i=1}^m \alpha_i A_i\right)^T \preceq \left(\sum_i \alpha_i^2\right)\left(\sum_{i=1}^m A_i A_i^T\right),
\end{equation}
where $A \preceq B$ signifies that $B - A$ is positive semidefinite.
\end{lem}
\begin{proof}
The proof is as in lemma 3.11 in \cite{kannan2012random}.  For all $i$ and $j$,
$$
(\alpha_j A_i - \alpha_i A_j)(\alpha_j A_i - \alpha_i A_j)^T \succeq 0.
$$
Thus
$$
\begin{aligned}
0 &\preceq \frac{1}{2}\sum_{i=1}^m\sum_{j=1}^m (\alpha_j A_i - \alpha_i A_j)(\alpha_j A_i - \alpha_i A_j)^T \\
&= \frac{1}{2} \sum_{i=1}^m \left[\left(\sum_{i=1}^m \alpha_j^2\right)A_i A_i^T  - \alpha_i A_i \sum_{j=1}^m \left(\alpha_j A_j^T\right) -  \left(\sum_{j=1}^m \left(\alpha_j A_j\right)\right)\left(\alpha_i A_i^T\right)  + \alpha_i^2 \sum_{i=1}^m A_j A_j^T\right]\\
&=  \left(\sum_i \alpha_i^2\right)\left(\sum_{i=1}^m A_i A_i^T\right) - \left(\sum_{i=1}^m \alpha_i A_i\right) \left(\sum_{i=1}^m \alpha_i A_i\right)^T.
\end{aligned}
$$
\end{proof} \noindent

Now we study how the volume and aspect ratio of the John's ellipsoid changes from a move from $x$ to $y$.  If the John's ellipsoid centered at $x = 0$ is the unit ball, and we make a move to $y$, the matrix $E_y$ such that  $$\mE_y = \{z| (z - y)^T E_y^{-2} (z - y) \leq 1\}$$ is the unique (from John's Theorem) $\arg\max$  over positive definite matrices of $- \log \det E$ under the constraint that
\begin{equation} \label{eq:maxvolatym} 
\mE_y \in \mK_y^s.
 \end{equation}

The John's ellipsoid of $\mK_x^s$ is a unit ball at the origin. Let us translate $\mK$ by $x-y$. All the contact points with $\mK + (x-y)$ of $\mE_x + (x-y)$ must lie on the boundary of $\mE_y$ or outside $\mE_y$. This condition can be rewritten as 
\begin{equation}\label{eq:3.6}  |E_y u_i| + \ip{u_i}{y} \leq 1,\quad i = 1, \ldots, m.\end{equation}

Note that we {\it do not} claim that $E_y$ is the matrix with largest determinant that satisfies the above constraints. There can be other constraints corresponding to contact points for $\mE_y$ that are not contact points of $\mE_x$.

Using Theorem \ref{thm:john}, and Lemma \ref{lem:semidef_cs}, we deduce an upper bound on $\det E_y$ as follows.

Recall that we will denote a universal positive constant that is large by $C$ and a universal positive constant that is small by $c$.

\begin{lem}\label{lem:det_upper}
Let $r = cn^{-5/2}$, and assume $y$ is chosen from a ball of radius $r$ such that $\|y - x\|_x = |y - x| \leq r$.  Then 
$$
\det E_y \leq 1 + 2 n^{-2}.
$$
\end{lem}
\begin{proof}
Note that by (\ref{eq:3.6}), $E_y$ satisfies the constraints $|E_y u_i| \leq 1 - u_i^T y.$ Since the weights $c_i$ corresponding to the John's ellipsoid $\mK_x^s$ are positive, the constraint implies that
$$
\sum_i c_i u_i^T E_y^2 u_i \leq \sum_i c_i(1- u_i^Ty)^2.
$$
By \eqref{eq:john_cond2}, \eqref{eq:john_csum}, and using the linearity and cyclic invariance of the trace, we have
$$
\begin{aligned}
\tr(E_y^2) &\leq \sum_i c_i - 2 \sum_i c_i u_i^T y + \sum_i c_i (u_i^T y)^2 \\
&= n - 2 \left(\sum_ic_i u_i\right)^T y + y^T \left(\sum_i c_i u_i u_i^T \right)y \\
&=  n - 2 \left(\sum_i c_i u_i \right)^T y + |y|^2.
\end{aligned}
$$
Considering $|\sum_i c_i u_i|$ to bound the middle term, we may employ Lemma \ref{lem:semidef_cs}.  Letting $\alpha_i = \sqrt{c_i}$ and $A_i = \sqrt{c_i} u_i$, we have
$$
\left(\sum_i c_i u_i\right)\left(\sum_i c_i u_i\right)^T \preceq \left(\sum_i c_i\right)\left(\sum_i c_i u_i u_i^T\right)
$$
Noting the right side is equal to $n I_n$, it follows that
$$
\bigg|\sum_i c_i u_i\bigg| \leq \sqrt{n}.
$$
Therefore, if $y$ is chosen from a ball of radius $cn^{-5/2}$, by Cauchy-Schwarz we conclude that
\begin{equation}
\tr(E_y^2) \leq n  + cn^{-2}.\label{eq:Eysquared}
\end{equation}
Now letting the eigenvalues of $E_y$ be denoted $d_i > 0 $, we have by the arithmetic-geometric mean inequality,
$$
\begin{aligned}
(\det E_y)^{2/n} &= \left(\prod_{i=1}^n d_i^2 \right)^{1/n} \\
&\leq \frac{1}{n}\sum_{i=1}^n d_i^2 \\
&\leq \frac{1}{n}\left(n + 2n^{-2}\right),
\end{aligned}
$$
Thus,
$$
\begin{aligned}
\det E_y &\leq&  (1 + cn^{-3})^{n/2} \\
& \leq&  1 + 2n^{-2}.
\end{aligned}
$$

\end{proof}

We deduce a lower-bound on $\det E_y$ by considering a positive definite matrix of the form $E = \beta(I - \alpha y y^T)$ such that the corresponding ellipsoid $y + \mE$ is contained in the unit ball $\mE_x$.  Note that such a matrix has eigenvalue $\beta(1 - \alpha |y|^2)$ of multiplicity 1 corresponding to unit eigenvector $y/|y|$, and eigenvalues $\beta$ of multiplicity $n - 1$ corresponding to any unit vector $z$ which is orthogonal to $y$.  
\begin{lem}\label{lem:feasible}$y + \mE \subseteq \mE_x \subseteq \mK$, and hence $\lambda(\mE_y) \geq \lambda(\mE).$
\end{lem}
\begin{proof} Assume after rescaling the $\mE_x$ is the unit ball.
We divide the points $u$ on the boundary of $\mE_x$  into two sets: $A = \set{u\st \ip{u}{y} \leq \frac{|y|}{\sqrt{n}}}$ and $B = \set{u \st \ip{u}{y}  > \frac{|y|}{\sqrt{n}}}$.  

If $u\in A$, we have
$$
1 - \ip{u}{y} \geq 1 - \frac{|y|}{\sqrt{n}},
$$
and noting that $E \succ 0$ and $0 \prec (I_n- \alpha y y^T) \prec I_n$,
$$
|E u| \leq \beta \leq 1 - \ip{u}{y}.
$$

 If $u \in B$, we have
$$
\begin{aligned}
|E u|^2 &= \beta^2 u^T(I - \alpha y y^T)^2 u\\
&\leq u^T(I - \alpha y y^T)^2 u \\
&= 1 - \alpha \ip{u}{y}^2\\
& < 1 - \ip{u}{y}.
\end{aligned}
$$
Thus $\mE + y \subseteq \mE_x \subseteq \mK$. Hence the volume of $\mE_y$ is at least the volume of $\mE$.  
\end{proof}

\begin{lem}\label{lem:det_lower}
Let $r = cn^{-5/2}$, and assume $y$ is chosen from a ball of radius $r$ such that $\|y - x\|_x = |y - x| \leq r$.  Then 
$$
\det E_y \geq 1 - 3 n^{-2}.
$$
\end{lem}
\begin{proof}
Considering the matrix $E$ as provided by Lemma \ref{lem:feasible}, $E_y$ satisfies 
$$
\begin{aligned}
\det E_y &\geq \det E  \\
&= \beta^n(1-\alpha|y|^2) \\
&= \left(1-\frac{|y|}{\sqrt{n}}\right)^n(1 - 2\sqrt{n}|y|).
\end{aligned}
$$
Thus 
$$
\begin{aligned}
\det E_y &\geq (1 - cn^{-2})(1 - 2cn^{-2}) \\
&\geq (1 - 3 n^{-2}).
\end{aligned}
$$
\end{proof}

Lemmas \ref{lem:det_upper} and \ref{lem:det_lower} establish that for some universal constant $c > 0$ and $|y - x| \leq cn^{-5/2}$, the volume ratio of $\mE_y$ and $\mB_x$ satisfies
\begin{equation} \label{eq:vol_ratio}
1 - 3n^{-2} \leq \frac{\lambda(\mE_y)}{\lambda(\mB_x)} \leq 1 + 2n^{-2},
\end{equation}
This does not necessarily indicate that the shapes of $\mE_x$ and  $\mE_y$ are close,  a property we require so rejection does not occur too frequently.  The following lemma guarantees that this is indeed the case.

\begin{lem}\label{lem:min_eigen}
Let $d_1 \geq d_2 \geq \ldots \geq d_n > 0$ denote the eigenvalues of $E_y$.  Then  the minimum eigenvalue $d_n$ satisfies
$$
d_n \geq 1 - 4n^{-1}
$$ 
for large enough $n$. Similarly,  the maximum eigenvalue $d_1$ satisfies
$$
d_1 \leq 1 + Cn^{-1}
$$ 
for large enough $n$.
\end{lem}
\begin{proof}
Assume the eigenvalues are ordered such that $d_1 \geq \ldots \geq d_n$.  By Lemma~\ref{lem:det_lower} we see that
$$
\det E_y  = \prod_{i=1}^n d_i \geq 1 - 3n^{-2}.
$$
By the power mean inequality and $\tr(E_y^2) \leq n + c/n^2$ (see (\ref{eq:Eysquared})), it follows that
$$
\begin{aligned}
\frac{1}{n} \tr(E_y) &\leq \sqrt{\frac{1}{n}\tr(E_y^2)} \\
&\leq 1 + n^{-3},
\end{aligned}
$$
so $\tr(E_y) \leq n + n^{-2}.$  By the arithmetic-geometric mean inequality, we thus have
$$
\begin{aligned}
(1 - 3n^{-2})/d_n &\leq \prod_{i < n} d_i \\
&\leq \left(\frac{\sum_{i< n}d_i}{n-1}\right)^{n-1} \\
&\leq \left(\frac{n + n^{-2} - d_n}{n-1}\right)^{n-1} \\
&= \left(1 + \frac{1+n^{-2}-d_n}{n-1}\right)^{n-1}.
\end{aligned}
$$
Since $(1 + z/m)^m \leq \exp(z)$ for $z \geq 0$ and integer $m \geq 1$, 
$$
(1 - 3/n^2) \leq \exp(1 + n^{-2} - d_n) d_n,
$$ 
The claim is true if $d_n \geq 1$, so we may assume that $\zeta := 1 - d_n > 0$. 

Then, $$\exp(\zeta + n^{-2})(1 - \zeta) \geq (1 - \frac{3}{n^2}).$$
This implies that 
$$\exp(\zeta)(1 - \zeta) \geq (1 - \frac{6}{n^2}).$$ 
We note that $\exp(\zeta)(1 - \zeta)$ is a monotonically decreasing function of $\zeta$ on $(0, 1)$, because its derivative is 
$$-\exp(\zeta) + \exp(\zeta)(1 - \zeta) = - \zeta \exp(\zeta).$$
If $\zeta = \Omega(1)$, then $\exp(\zeta)(1- \zeta) \leq 1 - \Omega(1).$ Therefore, we assume that $\zeta = o(1)$. But then, 

$$1 - \frac{\zeta^2(1 - \zeta)}{2} \geq \exp(\zeta)(1 - \zeta) \geq 1 - \frac{6}{n^2}$$ gives us 
$\zeta \leq \frac{4}{n}$, which implies the lower  bound on $d_n$.

To obtain the desired upper bound on $d_1,$ we first note that since there is a uniform lower bound of $1- \frac{4}{n}$ on all the $d_i$ and as the product lies within $(1 - \frac{3}{n^2}, 1 + \frac{2}{n^2})$, we obtain a uniform upper bound of $C$ on the $d_i$ and hence an upper bound on $C$ on $d_1$. This means that we may reverse the roles of $x$ and $y$, and noting that if $y - x \in \mE_x(r)$ then $x - y \in \mE_x(Cr)$ obtain the following upper bound on $d_1$.
$$d_1 \leq 1 + \frac{C}{n^2}.$$
\end{proof}

Now to derive a lower bound on the conductance for John's walk, we first must bound the statistical distance between two points given a bound on their geometric distance with respect to the local norm.  Again without loss of generality in what follows we may assume $x = 0$ and the John's ellipsoid centered at $x$ is the unit ball $\mB_x$ (otherwise perform an affine transformation such that this is the case).  Let $x, y \in \mK$ represent any two points in the body such that $\|y - x\|_x = |y - x| \leq r$, where $r \in (0, 1)$ is a constant to be specified in terms of the dimension $n$.  Let $P_x$ and $P_y$ denote the one-step transition probability measures defined at $x$ and $y$, respectively.  Let the uniform probability measures defined by the rescaled John's ellipsoids $\mE_x(r)$ and $\mE_y(r)$ be denoted $Q_x$ and $Q_y$, respectively.  We seek to bound
\begin{equation}\label{eq:tv_all_terms}
d_{TV}(P_x, P_y) \leq d_{TV}(P_x, Q_x) + d_{TV}(Q_x, Q_y) + d_{TV}(Q_y, P_y)
\end{equation}
by choosing $r$ such that the right side of \eqref{eq:tv_all_terms} is $1-\Omega(1)$.

To bound $d_{TV}(Q_x, Q_y)$ in \eqref{eq:tv_all_terms}, letting $\widetilde{Q}_y$ denote the probability measure corresponding to the uniform distribution on a ball of radius $r$ centered at $y$, we may alternatively bound
\begin{equation}\label{eq:tv_mid_term}
d_{TV}(Q_x, Q_y) \leq d_{TV}(Q_x, \widetilde{Q}_y) + d_{TV}(\widetilde{Q}_y, Q_y).
\end{equation}
We bound each term in \eqref{eq:tv_mid_term} separately.  To bound $d_{TV}(Q_x, \widetilde{Q}_y)$, note that by our assumption that $\mE_x = \mB_x$ (the unit ball at $x$), the corresponding densities with respect to the dominating Lebesgue measure $\lambda$ are 
$$
q_x(z) = \left(\frac{1}{\lambda(\mB_x(r))}\right) \cdot 1_{\set{z \in \mB_x(r)}}
$$
and 
$$
\widetilde{q}_y(z) = \left(\frac{1}{\lambda(\mB_y(r))}\right) \cdot 1_{\set{z \in \mB_y(r)}}.
$$
Thus using \eqref{eq:tv_density} and noting $\lambda(\mB_x(r)) = \lambda(\mB_y(r))$, we have
\begin{equation}\label{eq:tv_balls}
\begin{aligned}
d_{TV}(Q_x, \widetilde{Q}_y) &= 1 - \int_{\mB_x(r) \cap \mB_y(r)} q_x(z)\, d\lambda(z) \\
&= 1 - \frac{\lambda\left(\mB_x(r) \cap \mB_y(r)\right)}{\lambda(\mB_x(r))} \\
&= 1 - \frac{\lambda\left(\mB_x \cap \mB_y\right)}{\lambda(\mB_x)}.
\end{aligned}
\end{equation}
The Lebesgue measure of $\mB_x \cap \mB_y$ is equal to twice the volume of a spherical cap.  The following lemma regarding the volume of a hyperspherical cap from \cite{lovasz1993random} is useful.
\begin{lem}\label{lem:capvol}
Let $\mB_x \subset \R^n$ be the Euclidean ball of unit radius centered at $x$.  Let $\mH \subset \R^n$ define a halfspace at a distance of at least $t$ from $x$ (so $x$ is not contained in the halfspace).  Then for $t \leq \frac{1}{\sqrt{n}}$, we have
$$
\lambda(\mH\cap \mB_x) \geq \frac{1}{2}\left(1 - t\sqrt{n}\right) \lambda(\mB_x).
$$
\end{lem}
\noindent The following lemma results trivially from lemma \ref{lem:capvol} and \eqref{eq:tv_balls}.
\begin{lem}\label{lem:tv_balls}
Let $t \leq 1$.  If $\|y - x\|_x = |y-x| \leq \frac{rt}{\sqrt{n}}$, then $$
d_{TV}(Q_x, \widetilde{Q}_y) \leq t.
$$
\end{lem}

To bound $d_{TV}(\widetilde{Q}_y, Q_y)$, note that we are bounding the total variation distance between a density supported on a ball and a density supported on an ellipsoid with the same center.  The following lemma provides the bound.
\begin{lem}\label{lem:tv_qandqtilde}
If $\|y-x\|_x \leq r = cn^{-5/2}$, the total variation distance between $\widetilde{Q}_y$ and $Q_y$ satisfies
$$
d_{TV}(\widetilde{Q}_y, Q_y) \leq 1/4.
$$
\end{lem}
\begin{proof}
Note that by  \eqref{eq:tv_density}, we have 
$$
\begin{aligned}
d_{TV}(\widetilde{Q}_y, Q_y) &= 1 - \E_{\widetilde{Q}_y}\left[\min \left(1, \, \frac{\lambda(\mB_y(r))}{\lambda(\mE_y(r))} \right)\right]\\
&= 1- \min\left[1,\, \frac{\lambda(\mB_y)}{\lambda(\mE_y)}\right] \P_{\widetilde{Q}_y}(Z \in \mE_y(r)) \\
&\leq 1- \min\left[1,\, \frac{\lambda(\mB_y)}{\lambda(\mE_y)}\right] \P_{\widetilde{Q}_y}(Z \in \mE_y(r)|B)\P_{\widetilde{Q}_y}(Z \in B),
\end{aligned}
$$
where $B$ denotes the event in which $Z \in (1 - \frac{C}{n})\cdot\mB_y(r)$.  
By lemma \ref{lem:min_eigen}, it follows that $\P_{\widetilde{Q}_y}(Z \in \mE_y(r)|B) = 1$ since the smallest eigenvalue of $E_y$ is at least $1 - Cn^{-1}$. Additionally by \eqref{eq:vol_ratio},
$$
\begin{aligned}
\min\left[1, \, \frac{\lambda(\mB_y)}{\lambda(\mE_y)}\right] &\geq \frac{1}{1+Cn^{-2}} \\ 
&\geq \exp(-Cn^{-2}).
\end{aligned}
$$
Now noting that $(1-\frac{x}{2}) \geq e^{-x}$ for $x \in [0, 1]$, we have $\P_{\widetilde{Q}_y}(Z \in B) = \left(1-\frac{C}{n}\right)^n \geq e^{-2c}$,  and
$$
d_{TV}(\widetilde{Q}_y, Q_y) \leq 1 - \exp\left(-C(2 + n^{-2})\right).$$
\end{proof}

To bound $d_{TV}(P_x, Q_x)$, we provide the following lemma.
\begin{lem}\label{lem:tv_pq}
If $\|y-x\|_x \leq r = cn^{-5/2}$, the total variation distance between $P_x$ and $Q_x$ satisfies
$$
d_{TV}(P_x, Q_x) \leq 1/4.
$$
\end{lem}
\begin{proof}
With some abuse of notation with regards to \eqref{eq:proposal_density}, temporarily let the density of $Q_x$ with respect to the dominating measure $\mu$ as defined by \eqref{eq:mu} be 
$$
q_x(y) \equiv \frac{dQ}{d\mu}(y) =  \left(\frac{1}{\lambda(\mE_x(r))}\right) \cdot 1_{\set{y \in \mE_x(r), \, y \neq x}}.
$$
Then since $q_x(x) = 0$ and $p_x(y) \leq q_x(y)$ for $y \neq x$, by \eqref{eq:tv_density} we have
$$
\begin{aligned}
d_{TV}(P_x, Q_x) &= 1 - \int_{\set{y \st q_x(y) \geq p_x(y)}} \min\left[q_x(y), \, p_x(y) \right] d\mu(y) \\
&= 1 - \int_{\mK\setminus\set{x}} \min\left[q_x(y), \, p_x(y) \right] d\lambda(y) \\
&= 1 - \int_{\mE_x(r) \cap \setminus\set{x}} \min \left[\frac{1}{\lambda(\mE_x(r))}, \, \frac{1}{\lambda(\mE_y(r))} \right]\cdot 1_{\set{x \in \mE_y(r)}}\, d\lambda(y) \\ 
&= 1 - \int_{\mE_x(r)}\left[\min\left(1, \, \frac{\lambda(\mE_x(r))}{\lambda(\mE_y(r))} \right)\right]\cdot 1_{\set{x \in \mE_y(r)}}\cdot \left(\frac{1}{\lambda(\mE_x(r))}\right)\, d\lambda(y) \\
&= 1 - \E_{Q_x}\left[\min\left(1, \, \frac{\lambda(\mE_x)}{\lambda(\mE_Y)} \right) \cdot 1_{\set{Y \in A}}\right].
\end{aligned},
$$
where we let $A$ denote the ``accept'' event in which $x \in \mE_Y(r)$.  Since $\mE_x = \mB_x$, as in the proof to lemma \ref{lem:tv_qandqtilde}, we have for all $Y \in A$
$$
\min\left[1, \, \frac{\lambda(\mE_x)}{\lambda(\mE_Y)}\right] \geq \frac{1}{1+Cn^{-2}},
$$
Therefore,
$$
\begin{aligned}
d_{TV}(P_x, Q_x) &\leq 1 - \left(\frac{1}{1+Cn^{-2}}\right)\P_{Q_x}(Y \in A) \\
&\leq 1 - \left(\frac{1}{1+Cn^{-2}}\right)\P_{Q_x}(Y \in A|Y \in B)\P_{Q_x}(Y \in B),
\end{aligned}
$$
where $B$ is the event in which $Y \in (1-\frac{C}{n})\cdot\mE_x(r) = r(1-\frac{C}{n})\cdot\mB_x$.  Again by lemma \ref{lem:min_eigen}, $\P_{Q_x}(Y \in A| Y \in B) = 1$.  The remainder of the proof is as in lemma \ref{lem:tv_qandqtilde}.  
\end{proof} 

Note that by a similar argument, $d_{TV}(Q_y, P_y) \leq 1/4$ for some universal $c > 0$ as well.  Combining this with lemmas \ref{lem:tv_balls}, \ref{lem:tv_qandqtilde}, and \ref{lem:tv_pq}, the following theorem results.
\begin{thm}\label{thm:tv_all}
If $\|y-x\|_x \leq \frac{rt}{\sqrt{n}} = ctn^{-3}$ for some universal constant $c > 0$ and some $t \leq 1$, the total variation distance between $P_x$ and $P_y$ satisfies
$$
d_{TV}(P_x, P_y) \leq  3/4 + t = 1 - \epsilon.
$$
In particular, we may choose $t = 1/8$ so $\epsilon = 1/8$.  
\end{thm}
We finally arrive at a lower bound on the conductance for John's Walk using theorems \ref{thm:isoperimetry}, \ref{thm:norm_bound}, and \ref{thm:tv_all}.  The proof of the next result is similar to corollary 10 and theorem 11 in \cite{lovasz1999hit}.
\begin{thm}[Conductance Lower Bound]
Consider the partition $\mK = S_1 \cup S_2$ where $S_1, S_2 \in \mA$, and let $\pi$ be the uniform measure on $\mK$, i.e., 
$$
\pi(A) = \frac{\lambda(A)}{\lambda(\mK)} \quad \text{ for all } A \in \mA.
$$
Then for large enough $n$ and $t = 1/8$, we have
$$
\int_{S_1} P_x(S_2) d\pi(x) \geq \left(\frac{c}{512n^{7/2}}\right) \min(\pi(S_1), \pi(S_2)),
$$
so $\phi = \Omega(n^{-7/2})$.
\end{thm}
\begin{proof}
Note that the Radon-Nikodym derivative of $P_x$ with respect to the Lebesgue measure $\lambda$ is is well-defined for all $y \in \mK\setminus\set{x}$, and is given as
$$
\frac{dP_x}{d\lambda}(y) = \min \left[\frac{1}{\lambda(\mE_x(r))}, \, \frac{1}{\lambda(\mE_y(r))} \right] 1_{\set{x \in \mE_y(r),\, y \in \mE_x(r)}}.
$$
Let
$$
\rho(x) \equiv \frac{d\pi}{d\lambda}(x) = \frac{1}{\lambda(\mK)} \cdot 1_{\set{x \in \mK}}
$$
be the density for $\pi$.  Then for any $x, y \in \mK$ such that $y \neq x$, we have
$$
\rho(x)\frac{dP_x}{d\lambda}(y) = \rho(y) \frac{dP_y}{d\lambda}(x),
$$
from which it follows that $\pi$ is the stationary measure for the chain.  

Now consider points far inside $S_1$ that are unlikely to cross over to $S_2$.  Letting $t = 1/8$ so $\epsilon = 1/8$ as in theorem \ref{thm:tv_all}, we define
$$
S_1' \equiv S_1 \cap \set{x \st \rho(x) P_x(S_2) < \frac{\epsilon}{2\lambda(\mK)}}.
$$
Similarly, let 
$$
S_2' \equiv S_2 \cap \set{y \st \rho(x) P_y(S_1) < \frac{\epsilon}{2\lambda(\mK)}}.
$$
Since $\rho(x)P_x(S_2) \geq \epsilon/(2\lambda(\mK))$ for $x \in S_1\setminus S_1'$, we have
$$
\begin{aligned}
\int_{S_1} P_x(S_2) d\pi(x) &\geq \int_{S_1 \setminus S_1'} \rho(x) P_x(S_2) d\lambda(x) \\
&\geq \frac{\epsilon \lambda(S_1\setminus S_1')}{2\lambda(\mK)} \\
&= (\epsilon/2)\pi(S_1\setminus S_1').
\end{aligned}
$$
Similarly for $y \in S_2\setminus S_2'$, we have 
$$
\int_{S_2} P_y(S_1) d\pi(x) \geq (\epsilon/2)\pi(S_2\setminus S_2').  
$$
By the reversibility of the chain, we have
\begin{equation}\label{eq:reversibility}
\int_{S_1} P_x(S_2) d\pi(x) = \int_{S_2} P_y(S_1)d\pi(y),
\end{equation}
so it follows that
$$
\begin{aligned}
\int_{S_1} P_x(S_2)d\pi(x) &= \frac{1}{2}\int_{S_1} P_x(S_2)d\pi(x) + \frac{1}{2}\int_{S_2} P_x(S_1)d\pi(x) \\
&\geq \frac{\epsilon}{4}\left(\pi(S_1\setminus S_1') + \pi(S_2\setminus S_2')\right) \\
&= (\epsilon/4)\pi(\mK\setminus(S_1' \cup S_2')).
\end{aligned}
$$
Now let $\delta = ctn^{-3}$.  Assuming that $\pi(S_1') \leq (1-\delta)\pi(S_1)$, we have $\pi(S_1\setminus S_1') = \pi(S_1) - \pi(S_1') \geq \delta\pi(S_1)$, and thus
$$
\begin{aligned}
\int_{S_1} P_x(S_2) d\pi(x) &\geq \epsilon\delta\pi(S_1) \geq (\epsilon\delta/2) \min(\pi(S_1), \pi(S_2)) \\
&= \left(\frac{c}{128n^{3}}\right) \min(\pi(S_1), \pi(S_2)),
\end{aligned}
$$
which proves the claim.  Similarly if $\pi(S_2') \leq (1-\delta)\pi(S_2)$, the claim is proved again using \eqref{eq:reversibility}.  Thus assume that $\pi(S_1') > (1-\delta)\pi(S_1)$ and $\pi(S_2') > (1-\delta)\pi(S_2)$.  By theorem \eqref{thm:isoperimetry}, we have
$$
\pi(\mK \setminus(S_1' \cup S_2')) \geq \sigma(S_1', S_2')\pi(S_1')\pi(S_2').
$$
Now given $x \in S_1'$ and $y \in S_2'$, the total variation between $P_x$ and $P_y$ satisfies
$$
\begin{aligned}
d_{TV}(P_x, P_y) &\geq P_x(S_1) - P_y(S_1) \\
&= 1 - P_x(S_2) - P_y(S_1) \\
&\geq 1 - P_x(S_2') - P_y(S_1') \\
&> 1 - \epsilon.
\end{aligned}
$$
By theorem \ref{thm:tv_all}, it follows that $\|y-x\|_x > \delta$.  Then by theorem \ref{thm:norm_bound}, it follows that 
$$
\sigma(S_1', S_2') \geq n^{-1/2}\|y-x\|_x > \delta n^{-1/2}.
$$
Finally, we deduce that
$$
\begin{aligned}
\int_{S_1} P_x(S_2)d\pi(x) &\geq (\epsilon\delta n^{-1/2}/4) \pi(S_1')\pi(S_2') \\
&\geq \left(\frac{\epsilon \delta(1-\delta)^2}{4\sqrt{n}}\right)\pi(S_1)\pi(S_2) \\
&\geq \left(\frac{\epsilon \delta(1-\delta)^2}{8\sqrt{n}}\right)\min(\pi(S_1), \pi(S_2))\\
&= \left(\frac{c(1-\delta)^2}{512n^{7/2}}\right)\min(\pi(S_1), \pi(S_2)).
\end{aligned}
$$
The claim follows by absorbing terms into the constant for large enough $n$.
\end{proof}
Now applying Corollary~\ref{cor:mixing}, we obtain our main theorem:

\begin{thm}
For $\epsilon > 0$ and $M \geq  \sup \frac{\pi_0(A)}{\pi(A)}$, after $t(\epsilon) =  C n^7 \log({M}/\epsilon) $ steps of John's walk, we have $d_{TV}(\pi_{t(\epsilon)}, \pi) \leq \epsilon$.  
\end{thm}

As a matter of fact, this theorem allows us to find mixing time bounds starting from any point that is not on the boundary of $\mK$.
 Suppose we know that $x$ belongs to the interior of $\mK$ and satisfies the following chord condition. For all chords $pq$ of $\mK$ containing $x$, assume $\frac{|p-x|}{|q-x|} \in (\eta, \eta^{-1})$ for some parameter $0 < \eta < 1$ that measures the centrality of $x$ in $\mK$. Then, we see that $\lambda(\mB_x(r)) \geq \left(\frac{r\eta}{\sqrt{n}}\right)^n  \lambda(\mK).$ After a random  geometrically distributed time $\tau$ with mean bounded above by an absolute constant, the first nontrivial move occurs. Then the distribution of $x_\tau$ has a density bounded above by $M = \left(\frac{r\eta}{\sqrt{n}}\right)^{-n}.$ We thus have the following theorem.

\begin{thm}
For all chords $pq$ of $\mK$ containing $x$, assume $\frac{|p-x|}{|q-x|} \in (\eta, \eta^{-1})$ for some parameter $0 < \eta < 1$ that measures the centrality of $x$ in $\mK$. Then, there is a random geometrically distributed time $\tau$ with mean bounded above by $C$ such that 
for $\epsilon > 0$, after $t(\epsilon) + \tau = C n^7 \left(n \log(\sqrt{n}/(r\eta)) + \log({1}/\epsilon)\right) + \tau$ steps of John's walk, we have $d_{TV}(\pi_{t(\epsilon) + \tau}, \pi) \leq \epsilon$.  

\end{thm}

\section{Conclusion}

We introduced an affine-invariant random walk akin to Dikin Walk which uses uniform sampling from John's ellipsoids of a certain small radius of appropriately symmetrized convex sets to make steps, and showed that this walk mixes to within a total variation distance $\epsilon$ in $O(n^7\log \epsilon^{-1})$ steps from a warm start.   The type of convex body $\mK$ is not specified (i.e., need not be a polytope) in our analysis of the mixing time, but one must have access to the John's ellipsoid of the current symmetrization of the convex body.  A significant  feature of this walk is that its  mixing time from a warm start or alternatively ``central point" such as the center of mass, can be bounded above by a quantity that has absolutely no dependence on any parameter associated with the body apart from its dimension. 

\section*{Acknowledgements}

HN was partially supported by a Ramanujan Fellowship and NSF award no. 1620102.
HN also acknowledges the support of  DAE project no. 12-R\&D-TFR-5.01-0500.

\bibliographystyle{alpha}
\bibliography{refs}
\end{document}